\documentclass{article}

\usepackage[preprint]{neurips_2024}

\usepackage[utf8]{inputenc}
\usepackage[T1]{fontenc}
\usepackage{amsmath,amssymb,amsthm}
\usepackage{mathtools}
\usepackage{booktabs}
\usepackage{graphicx}
\usepackage{microtype}

\theoremstyle{plain}
\newtheorem{theorem}{Theorem}
\newtheorem{lemma}{Lemma}
\newtheorem{proposition}{Proposition}

\theoremstyle{definition}
\newtheorem{definition}{Definition}

\newcommand{\Bdir}{B_{\mathrm{dir}}}
\newcommand{\Bprec}{B_{\mathrm{prec}}}
\newcommand{\hcap}{h_{\mathrm{cap}}}
\newcommand{\hexp}{h_{\mathrm{exp}}}
\newcommand{\Pid}{\Pi_d}
\newcommand{\Conf}{\mathrm{Conf}}
\newcommand{\supp}{\mathrm{supp}}
\newcommand{\req}{\mathrm{req}}
\newcommand{\gain}{\mathrm{gain}}
\newcommand{\KL}{\mathrm{KL}}
\newcommand{\OPT}{\mathrm{OPT}}

\title{Capability-Gated Planning: Cost-to-Goal Discovery and the\\ Limits of Myopic Experiment Selection}

\author{%
  Ahmed Hassoon\\
  Johns Hopkins University\\
  \texttt{ahassoo1@jhu.edu}\\
  \And
  Mark Dredze\\
  Johns Hopkins University\\
  \texttt{mdredze1@jh.edu}\\
}

\begin{document}
\maketitle

\begin{abstract}
Systems that automate scientific discovery must repeatedly decide which experiment to run, which hypothesis to test, which tool or representation to build, and when to stop. Many systems make these decisions by maximizing a myopic score, such as expected information gain (EIG) per unit cost or a learned plausibility or utility score. We identify a structural limitation of this approach. Some actions are \emph{constructive}: they acquire an epistemic capability---a physical instrument, calibrated assay, data pipeline, simulator, reusable abstraction, or other prerequisite---whose value lies not in the information returned immediately but in the future actions it makes available. In the formal setting studied here, constructive experiments unlock downstream measurements. When the least-cost route to a confident answer requires a chain of such constructions, a planner that scores actions only by information obtainable within a bounded horizon cannot value the initial construction. That action yields no information within the horizon and is therefore dominated by any measurement with positive information, however small.

We formulate goal-directed discovery as a stochastic shortest-path (SSP) problem in belief space, with constructive experiments represented as edges that change the downstream action graph. We prove that, for every fixed lookahead depth $d$, there is an instance on which every planner in the myopic information-maximizing class has an unbounded approximation ratio relative to the optimum (Theorem~\ref{thm:unbounded}); on a related instance, such a planner may never reach the goal (Proposition~\ref{prop:fail}). The key result is a capability-indistinguishability lemma: within a $d$-step horizon, acquiring a capability can be observationally indistinguishable from paying for a null action. Chain depth determines whether the capability route is visible, and signal strength determines the cost of remaining on the visible route. This establishes capability gating as a reachability axis of difficulty distinct from curvature (submodularity) and information order (adaptivity gaps). We then introduce \textbf{CG-Plan} (Capability-Gated Planning), an incremental replanner with the capability-aware cost-to-go heuristic $h = \hcap + \hexp$. In a controlled testbed, the performance gap appears only under gating, persists for every fixed lookahead horizon, and also arises when near-miss hypotheses are generated by a data-consistent proposer rather than constructed manually.
\end{abstract}

\section{Introduction}

Automating scientific discovery requires repeated decisions about which experiment to run, which hypothesis to test, which capability to acquire, and when to stop. Across sequential experimental-design loops, ``AI co-scientists,'' and self-driving laboratories, these decisions are often made by maximizing a local score: expected information gain about the target, a learned plausibility or utility score, or the outcome of a tournament among proposed hypotheses. Such a rule is myopic in a precise sense: the value assigned to a candidate depends on the distribution of outcomes that the candidate, and perhaps a small number of follow-up actions, is expected to produce.

This approach is effective when every useful action is already available as a measurement. It can fail when some actions are \emph{constructive}. Building a microscope does not itself reveal information about a cell; it changes what can be measured afterward. Calibrating an assay, deriving a reagent, establishing a simulation pipeline, obtaining access to a linked dataset, or constructing a reusable abstraction may similarly impose a one-time cost that changes the set of downstream epistemic actions. The least-cost route to a confident answer may therefore require acquiring a prerequisite capability before the decisive question can be measured, tested, expressed, or even posed.

The formal analysis considers the simplest version of this setting: constructive experiments that unlock downstream measurements. A myopic selector cannot assign value to the first construction when its score depends only on immediate or bounded-horizon information. Because the construction yields no information within that horizon, its information-per-cost score is zero and any available measurement with positive information is ranked above it. The selector may then continue taking inexpensive, low-yield measurements rather than acquiring the capability needed to resolve the query. This failure is structural rather than a matter of tuning: the value of construction is a change in the future feasible action set, not a property of an observed outcome distribution within the scoring horizon.

The same principle extends beyond physical instruments. A capability may be computational, institutional, methodological, procedural, or representational: examples include a calibrated protocol, data-access pipeline, simulator, robot skill, program library, or primitive that makes a new hypothesis class expressible. We do not formalize this broader open-ended setting. The theorem is intentionally restricted to gated measurements because they isolate the reachability mechanism cleanly. Nevertheless, the broader interpretation is useful: constructive actions may expand the set of future epistemic operations, not only the set of laboratory measurements.

We formalize this failure mode and derive a planner that explicitly accounts for capability acquisition.

\paragraph{Contributions.}
\begin{enumerate}
\item \textbf{A formulation (\S\ref{sec:setting}--\ref{sec:ssp}).} We formulate goal-directed discovery as a stochastic shortest-path problem in belief space. Nodes are epistemic states, edges are experiments or other discovery actions, and the objective is to minimize expected cost until a confidence target is reached. Constructive experiments induce one-time changes in downstream action availability or cost. Their value is therefore a property of graph topology. Classical sequential experimental design is recovered as the special case with no constructive edges.
\item \textbf{A reachability axis and separation theorem (\S\ref{sec:axes}--\ref{sec:separation}).} We define the class $\Pid$ of myopic planners whose action values factor through within-horizon observation distributions. This class includes sequential BOED and any outer selection rule based on a bounded-horizon information or utility functional of predicted outcomes. Under capability gating, every planner in this class can be arbitrarily suboptimal (Theorem~\ref{thm:unbounded}); on a related instance, it can fail to reach the goal (Proposition~\ref{prop:fail}). The capability-indistinguishability lemma isolates the mechanism. Unlike negative results based on curvature or information order, the resulting approximation gap is unbounded along a reachability dimension that those theories do not parameterize.
\item \textbf{A capability-aware planner, CG-Plan (\S\ref{sec:cgplan}).} The separation motivates a cost-to-go heuristic that accounts for capabilities beyond the lookahead horizon without enumerating an exponentially large search tree. We define $h = \hcap + \hexp$ from a delete-relaxation of the SSP formulation. In the chain-gated setting, this relaxation is admissible; beyond that setting, a path-aware variant compares the build and direct routes. The capability term $\hcap$ depends on the capability graph rather than on an outcome distribution, so it lies outside $\Pid$. CG-Plan places this heuristic in an incremental-replanning loop.
\item \textbf{A controlled empirical illustration (\S\ref{sec:experiments}).} In a controlled testbed, CG-Plan exhibits the behavior predicted by the theory. The gap appears and disappears with the gating parameter, persists for every fixed lookahead horizon with a sharp chain-length boundary, and remains when near-miss hypotheses are produced by a data-consistent proposer. Distractor-build and no-gating controls behave as predicted.
\end{enumerate}

\paragraph{Scope of the claims.} The theorem is a clarifying lower bound. Once the planner class is defined, the proof is direct, but the formulation isolates a failure mode that standard information-selection theory does not parameterize. The instance family and testbed are adversarial \emph{witnesses}; they do not establish that real scientific discovery problems commonly contain deep capability gates. We also do not prove that scientific discovery generally requires representational or hypothesis-generation capabilities. Those empirical questions require separate evidence, such as rediscovery backtests using historical corpora, and are outside the scope of this paper.

\section{Problem setting}\label{sec:setting}

A hidden system $M^*$, fixed within an episode, responds to experiments through a fixed interface. The agent maintains a hypothesis language $\mathcal{L}$ of \emph{executable} models. Executability is required so that (a) prediction failures can be localized, (b) consistency with the full record can be checked automatically, (c) interventional and counterfactual queries can be answered, and (d) candidate experiment sequences can be evaluated by rolling models forward. Structural causal models with parametric mechanisms, probabilistic programs, and component-based simulators satisfy these requirements. Free-text hypotheses are excluded because they are not directly executable.

We use \emph{experiment} broadly to denote any costly discovery action. An experiment $e$ has capability preconditions $\req(e)$, a positive scalar cost $c(e)$, granted capabilities $\gain(e)$, and an outcome law $P(o \mid e, M^*)$. The cost may represent money, wall-clock time, computation, staff effort, risk, regulatory burden, or a domain-specific weighted combination of these quantities. We use a scalar cost to obtain a shortest-path formulation; practical systems may instead optimize multiple resources subject to constraints.

A capability is a prerequisite that changes which downstream actions are feasible. It may be physical (an instrument), computational (a simulator or code library), procedural (a calibrated assay or protocol), institutional (data access), or representational (a primitive that enables a class of hypotheses or experiments). The present formalism models capabilities through action availability rather than through changes to the hypothesis language. State-dependent hypothesis and query languages are discussed as an extension in \S\ref{sec:discussion}.

Two kinds of experiments suffice for the formal result:
\begin{itemize}
\item \textbf{Interventional} experiments query $M^*$ and carry information about it ($\gain(e) = \emptyset$).
\item \textbf{Constructive} experiments grant capabilities ($\gain(e) \neq \emptyset$). Their outcomes are deterministic or nearly deterministic, and, by construction, carry zero or negligible information about $M^*$. Acquiring a capability does not reveal the hidden mechanism until the capability is used in a downstream experiment.
\end{itemize}

The target is a query $q$ with support $\supp(q)$, defined as the set of model components whose perturbation changes the answer. The agent must report an answer to $q$ at a calibrated confidence level. The objective is therefore not to learn the entire system, but to acquire sufficient information about the components relevant to $q$, possibly after constructing the capabilities needed to measure them.

\section{Discovery as cost-to-goal search in belief space}\label{sec:ssp}

\paragraph{Epistemic state.} The agent's state is $s = (B, D, I, v)$. Here, $B$ is a weighted ensemble of executable models, with weights derived from held-out predictive performance on the record and a diversity floor that prevents premature collapse of viable alternatives. The archive $D$ contains permanent experiment records and also serves as a regression suite. The capability set $I$ determines action availability: $e$ is available in state $s$ if and only if $\req(e) \subseteq I$. The counter $v$ records the state version.

\paragraph{Transition.} Executing $e$ in state $s$ and observing $o$ produces $s' = (B', D \cup \{(e,o)\}, I \cup \gain(e), v')$, where $B'$ is the belief refitted to the augmented record. The agent selects $e$ at an OR node, and nature samples $o$ at an AND node. The resulting belief-space model is therefore an AND/OR graph; \S\ref{sec:cgplan} describes the determinization used by CG-Plan.

\paragraph{Goal.} Given a query $q$ and tolerance $\varepsilon$,
\[ G = \{\, s : \Conf_s(q) \ge 1 - \varepsilon \ \text{ and } \ \text{anomaly mass on } \supp(q) \le \tau_G \,\}, \]
where $\Conf_s(q)$ denotes ensemble agreement on the answer: the weight assigned to the modal answer for discrete $q$, or the weight within a tolerance band of the weighted median for continuous $q$. We distinguish the stopping criterion from external evaluation. The agent stops when $s \in G$, a decision based only on the current state; an external evaluator records success only when the reported answer equals $M^*(q)$. Under model misspecification, the agent may stop confidently with an incorrect answer. A complete evaluation therefore reports both stopping time and correctness.

\paragraph{Objective.}
\[ \min_\pi\ J(\pi) = \mathbb{E}\Big[\textstyle\sum_t c(e_t)\Big] \ \text{ until } s_T \in G, \]
which is a stochastic shortest-path problem in belief space. Deterministic SSPs are addressed by A*, LPA*, and D*~Lite \citep{hart1968,koenig2004lpastar,koenig2002dstarlite}; stochastic SSP methods include LAO*, RTDP, and PPCP \citep{hansen2001lao,barto1995rtdp,likhachev2009ppcp}. CG-Plan builds on this established planning framework.

\paragraph{Capability gating as one-time enablement.} For a one-time cost, a constructive edge changes the edge set at all downstream states by making experiments available or reducing their cost. Its option value is therefore determined by graph topology rather than by an additive term in a local observation score. A planner with sufficient lookahead can account for this value, whereas a myopic per-experiment score generally cannot. The remainder of the paper formalizes this distinction.

In this state representation, capabilities are action-space objects: they determine which edges exist. This is the minimal structure required for the separation result. A richer model could also allow the hypothesis language $\mathcal{L}$ or query set to vary with the state, for example $s=(B,D,I,\mathcal{L},\mathcal{Q},v)$, so that constructive actions unlock new representations or questions. We return to this extension in \S\ref{sec:discussion}; the formal results below require only action-space gating.

\paragraph{BOED as a special case.} If all experiments are interventional, costs are uniform, and no actions are gated, the one-step decision reduces to classical sequential Bayesian experimental design: maximize expected information gain about $q$ per unit cost. Our claims concern settings excluded by this special case.

\section{Three axes of hardness}\label{sec:axes}

Greedy information gathering is already known to be suboptimal. To distinguish capability gating from existing negative results, we separate three properties of a sequential-selection problem.

\begin{table}[h]
\centering
\small
\begin{tabular}{@{}p{2.1cm}p{3.0cm}p{3.4cm}p{4.0cm}@{}}
\toprule
\textbf{Axis} & \textbf{What varies?} & \textbf{Typical theory} & \textbf{Why gating is different} \\
\midrule
Curvature & Marginal value of additional observations & Submodularity and non-submodular value of information & Assumes the observation set is fixed or already selectable \\
Information order & How much adaptivity helps as observations arrive & Adaptive submodularity, adaptivity gaps & Our failing planner can replan after every observation \\
Reachability & Whether the decisive action is available at all & Less directly parameterized in information-selection theory & The action that resolves the query may not exist until capabilities are built \\
\bottomrule
\end{tabular}
\end{table}

The first two axes are well studied. \emph{Curvature} concerns the value-of-information function over sets of observations: under submodularity, greedy subset selection is within $(1-1/e)$ of optimal, and under the non-submodular conditions studied by \citet{krauseguestrin2005}, the loss remains bounded by a constant factor. \emph{Information order} concerns the value of adapting actions to observations: under submodular or XOS structure, the adaptivity gap is bounded by a constant or logarithmic factor \citep{golovin2011,gns2017,bsz2019}. \emph{Reachability} concerns whether the decisive action is available within the planning horizon. The action that resolves the query may become available only after a chain of constructive actions, placing it beyond every fixed horizon chosen in advance.

Existing negative results on the first two axes yield bounded factors under their respective structural assumptions. The separation established here concerns the third axis and has an unbounded factor. The key distinction is that expected information gain is a functional of the predictive distribution of within-horizon observations, whereas a constructive experiment changes the feasible action set beyond that horizon. This change is not represented in any bounded-horizon observation functional. Consequently, neither cost reweighting nor a submodular relaxation can recover a value that is absent from the scoring functional itself. Section~\ref{sec:irreducibility} makes these distinctions explicit.

The reachability perspective also applies beyond measurements. The gated object may be an experiment, data source, simulator, skill, or representation. The theorem uses gated experiments because they permit a clean proof, but the underlying axis is epistemic reachability.

\section{The capability-gating separation}\label{sec:separation}

\subsection{The instance family}\label{sec:instance}

Fix a horizon bound $d \ge 1$. Instance $I(d, \gamma, c_b, \varepsilon)$ has a hidden bit $q \in \{0,1\}$ with a uniform prior, belief $p = \Pr[q=1]$, and goal $G = \{\max(p, 1-p) \ge 1-\varepsilon\}$. The objective is to minimize cost to the goal. The instance has three action types:
\begin{itemize}
\item \textbf{Direct probe} $\Bdir$, with cost $1$, is always available. It returns $o \in \{0,1\}$ with $\Pr[o = q] = (1+\gamma)/2$. This binary symmetric channel has per-call divergence $\KL = \gamma \ln\frac{1+\gamma}{1-\gamma} = \Theta(\gamma^2)$ as $\gamma \to 0$.
\item \textbf{Construction chain} $b_1, \dots, b_{d+1}$. Each action has cost $c_b > 0$, and $b_i$ requires completion of $b_{i-1}$. Every build succeeds deterministically and returns a $q$-independent observation, so it carries no information about $q$.
\item \textbf{Precision probe} $\Bprec$, with cost $1$, becomes available only after $b_{d+1}$ and returns $q$ exactly.
\end{itemize}
The construction chain has length $d+1$, which exceeds the lookahead horizon $d$. This inequality produces the separation.

\subsection{The myopic class}\label{sec:class}

Let $\Pid$ denote the class of policies that, after each observation, select
\[ \pi(s) \in \arg\max_{a \in A(s)} F_s\big(c(a),\, O^{\le d}_{s,a}\big), \]
where $O^{\le d}_{s,a}$ is the joint distribution of observations reachable within $d$ steps after choosing $a$, and $F_s$ is monotone in information about $q$ per unit cost. The functional may use capability-graph topology only through its effect on $O^{\le d}_{s,a}$; it has no separate term for capabilities that become useful beyond $d$ steps. The following definition makes the monotonicity condition precise.

\begin{definition}[$\Pid$, formal]\label{def:pid}
For costs $c, c' > 0$ and within-horizon observation ensembles $O, O'$---the joint laws of all observation sequences reachable in at most $d$ steps---write $O' \preceq_q O$ if $O'$ is a garbling of $O$ with respect to $q$. That is, there exists a channel $K$ such that $O' = K \circ O$, so $q \to O \to O'$ is a Markov chain. A functional $F_s$ is \emph{admissible for $\Pid$} if (i) it depends on a candidate action only through $(c, O)$; (ii) $O' \preceq_q O$ at equal cost implies $F_s(c, O') \le F_s(c, O)$; and (iii) it satisfies \emph{zero-information dominance}: if $O \perp q$ and $O''$ carries strictly positive information about $q$, then $F_s(c'', O'') > F_s(c, O)$ for all $c, c'' > 0$. The class $\Pid$ consists of policies that select $\arg\max_a F_s(c(a), O^{\le d}_{s,a})$ for some admissible $F_s$.
\end{definition}

Blackwell's theorem identifies $\preceq_q$ as the canonical information order, and any $F_s$ based on expected information gain per unit cost satisfies conditions (i)--(iii). Condition (iii) formalizes monotonicity in information per unit cost. It is the only condition needed for the strict ranking in Lemma~\ref{lem:indist}; the equality between the build and null actions uses only condition (i).

Membership in $\Pid$ is determined by the scoring functional, not by the surrounding system. A selector belongs to $\Pid$ if and only if its ranking criterion depends on the predicted within-horizon observation distributions and their costs as specified in Definition~\ref{def:pid}. Sequential BOED belongs to $\Pi_1$, and any design loop whose outer selection step maximizes an expected-information objective over predicted outcomes belongs to $\Pi_1$ regardless of how candidates are proposed. A plausibility- or utility-based selector belongs to $\Pid$ only when its learned score has this form. An unrestricted state-dependent score, including an LLM-based score, need not belong to the class because it may encode a preference for acquiring future capabilities that is not derived from a bounded-horizon observation functional. Our claims about deployed systems are therefore conditional: the theorem applies to documented information-based selection rules and to learned scores only insofar as they approximate such rules. The class also excludes planners with explicit long-horizon value, deep tree search, or a separate value term for capabilities beyond the horizon. Finally, $\Pid$ is a deterministic argmax class. Policies with undirected exploration lie outside it; \S\ref{sec:explore} extends the separation to such policies on a distractor-inflated instance family.

The quantifiers are important: the statement is $\forall d\, \exists\,\text{instance}$, not the reverse. A planner with horizon $d+2$ can solve a particular instance $I(d, \cdot)$. For every fixed horizon, however, an adversary can choose a construction chain that is one step longer.

\subsection{The lemma and the theorems}\label{sec:lemma}

\begin{lemma}[Capability indistinguishability]\label{lem:indist}
Let $r(s)$ be the number of constructive actions still required before $\Bprec$ becomes available, with $r(s_0) = d+1$ in the initial state. In any state satisfying $r(s) > d$, the next build $b_i$ and a fictitious null action $\oslash_{c_b}$ have the same cost and induce the same within-$d$-step observation ensemble; the null action pays $c_b$ and returns a $q$-independent observation. Every admissible functional $F_s$ therefore assigns equal value to these two actions. Because $\oslash_{c_b}$ carries no information about $q$ while $\Bdir$ carries strictly positive information, zero-information dominance ranks $\Bdir$ strictly above the build. In particular, no policy in $\Pid$ selects $b_1$ from $s_0$ while $\Bdir$ is available.
\end{lemma}

The full proof appears in Appendix~\ref{app:proofs}. The only action whose availability or outcome law changes as the chain advances is $\Bprec$. After the remaining $r(s)$ builds, one additional action is required to obtain its informative observation. If $r(s) > d$, no sequence of at most $d$ actions beginning with $b_i$ reaches that observation. The within-horizon observation ensemble after $b_i$ is therefore identical to the ensemble after $\oslash_{c_b}$. Since $\Bdir$ has positive information per unit cost and the null action has zero, condition (iii) of Definition~\ref{def:pid} gives the strict ranking for every $c_b > 0$.

The lemma does not claim that a build is uninformative in every state. After sufficient progress along the chain, the precision probe enters the horizon and the equality no longer holds. Only the first decision is needed: because a policy in $\Pid$ strictly prefers $\Bdir$ to $b_1$ at $s_0$, it never advances the chain and therefore never reaches a state in which the remaining builds are visible within the horizon. By induction, the policy remains confined to $\Bdir$. The capability's value lies in a change to the feasible action set beyond the horizon, not in a within-horizon observation distribution.

\begin{theorem}[Unbounded suboptimality]\label{thm:unbounded}
For every $d \ge 1$ and every $\rho > 1$, there exist $\gamma, c_b, \varepsilon$ such that on $I(d, \gamma, c_b, \varepsilon)$ every policy in $\Pid$ incurs expected cost at least $\rho \cdot \OPT$.
\end{theorem}

\emph{Proof sketch (full proof in Appendix~\ref{app:proofs}).} The optimal policy builds the chain and queries the precision probe, so $\OPT \le (d{+}1)c_b + 1$, and reaches confidence $1$ deterministically. By Lemma~\ref{lem:indist}, a policy in $\Pid$ never selects $b_1$ from the initial state. It therefore remains confined to $\Bdir$ and performs a sequential test between $q=0$ and $q=1$ through a channel with divergence $\Theta(\gamma^2)$. Wald's change-of-measure converse gives $\mathbb{E}[N] = \Omega(\log(1/\varepsilon)/\gamma^2)$ for any stopping rule that uses only direct probes. Hence the approximation ratio is $\Omega\big(\log(1/\varepsilon) / (\gamma^2 ((d{+}1)c_b+1))\big)$. Holding $d$, $c_b$, and $\varepsilon$ fixed while taking $\gamma \to 0$ makes the ratio arbitrarily large. \qed

The two instance parameters play distinct roles. The inequality $d+1>d$ makes the lower-cost capability route invisible to the bounded-horizon selector, while $\gamma$ determines the cost of resolving the query using only the visible direct probe. Both conditions are required: invisibility alone does not produce a large gap when the direct route is inexpensive.

\begin{proposition}[Failure to reach the goal under a no-build-on-ties convention]\label{prop:fail}
Modify the instance by drawing a nuisance bit $\nu$ once per episode and letting it confound the direct probe: $\Bdir$ reports $q \oplus \nu$ through the noisy channel, whereas $\Bprec$ still returns $q$ exactly. Consider a policy in $\Pid$ whose tie-breaking rule never prefers a zero-information constructive action to an equally valued non-constructive action, as in our experiments. On this capped instance, the policy does not reach $G$. Direct probes identify only $q \oplus \nu$, so the supremum confidence attainable about $q$ from direct probes is $1-\varepsilon' < 1-\varepsilon$. By Lemma~\ref{lem:indist}, the policy does not take the first build while direct probing has positive within-horizon value. It therefore remains below the confidence threshold regardless of how long it probes. An implementation that stops when residual EIG falls below a tolerance halts outside $G$; an implementation without such a stopping rule continues probing outside $G$. The optimal policy builds the chain, queries $\Bprec$, and reaches confidence $1$.
\end{proposition}

\paragraph{Role of the tie-breaking condition.} The condition is necessary. As the information supplied by direct probes approaches zero, the build chain and a non-constructive action can become equal-valued under a bounded-horizon observation score. A policy allowed to break such ties in favor of construction could then complete the chain and unlock $\Bprec$. Proposition~\ref{prop:fail} therefore depends on both bounded-horizon valuation and a tie-breaking rule that does not select zero-information constructions. We state it as a proposition for this reason. Theorem~\ref{thm:unbounded} requires no tie-breaking convention because $\Bdir$ strictly dominates the first build.

\subsection{Exploration does not restore boundedness}\label{sec:explore}

The class $\Pid$ uses deterministic argmax selection, whereas deployed systems may add undirected exploration, such as choosing a uniformly random action with probability $\varepsilon_0$ or sampling from a softmax distribution. On the base instance $I(d, \gamma, c_b, \varepsilon)$, such exploration eventually selects $b_1$ with positive probability, and permanent capabilities allow the chain to be completed in expected time independent of $\gamma$. The unbounded separation therefore does not hold on the base instance. It does hold on the following enriched family.

\paragraph{Distractor-inflated instance $I_N$.} This instance is identical to $I(d, \gamma, c_b, \varepsilon)$ except that each chain position $i \in \{1, \dots, d{+}1\}$ also has $N$ \emph{distractor} constructive actions $\delta_{i,1}, \dots, \delta_{i,N}$. Each distractor costs $c_b$, has the same precondition as $b_i$, and grants a capability required by no experiment. Distractors are repeatable, so the action set does not shrink. The optimum is unchanged: $\OPT \le (d{+}1)c_b + 1$.

\begin{definition}[$\varepsilon_0$-exploring $\Pid$ policy]\label{def:explore}
At each step, with probability $1 - \varepsilon_0$ select according to some $\Pid$ rule; with probability $\varepsilon_0$ select uniformly from the available actions.
\end{definition}

\begin{proposition}[Exploration does not restore boundedness]\label{prop:explore}
For every $d \ge 1$, $\varepsilon_0 \in (0, 1]$, and $\rho > 1$, there exist $\gamma, N, \varepsilon$ such that on $I_N(d, \gamma, c_b, \varepsilon)$ every $\varepsilon_0$-exploring $\Pid$ policy incurs expected cost at least $\rho \cdot \OPT$.
\end{proposition}

The proof sketch is given in Appendix~\ref{app:proofs}. By Lemma~\ref{lem:indist}, the $\Pid$ component never selects $b_1$ at the chain frontier. The first true build can therefore be selected only by exploration, with probability at most $\varepsilon_0/(N{+}2)$ per step. Independently, the direct-probe route requires $\Omega(\log(1/\varepsilon)/\gamma^2)$ samples. Taking $N \to \infty$ and $\gamma \to 0$ jointly makes the less expensive of these two routes arbitrarily costly relative to $\OPT$, which is independent of $N$. The same dilution applies to a Boltzmann policy over $\Pid$ values because the true build receives at most a $1/(N{+}1)$ share of the total build probability. Directed novelty or count-based exploration can reach the chain but cannot distinguish the required build from the distractors, so it may incur $\Theta(Nc_b)$ in unnecessary construction. Query-directed capability pricing avoids this failure by evaluating which capabilities are required for $q$: $\hcap$ uses $\mathrm{req\_meas}(\supp(q))$, and the distractor experiment in \S\ref{sec:explore-abl} shows that CG-Plan's build count remains equal to the true chain length as $N$ increases.

\subsection{Irreducibility}\label{sec:irreducibility}

\paragraph{Non-submodular value of information \citep{krauseguestrin2005,krauseguestrin2009}.} These results study subset selection over a fixed, fully available set of observations and do not include a lookahead horizon. Capability gating instead changes which actions exist in a state. Their approximation gaps are bounded under the stated assumptions, whereas the present gap is unbounded in $\gamma$ and the capped instance can fail to reach the confidence target. The parameters responsible for the separation---chain depth beyond a fixed horizon and the $1/\gamma^2$ direct-probe cost---have no counterparts in that formulation.

\paragraph{Adaptive submodular maximization \citep{golovin2011}.} Adaptive greedy obtains a $(1-1/e)$ approximation when the objective is adaptive submodular and the ground set is fixed and fully selectable. Capability gating violates the availability premise because $\Bprec$ cannot be selected until the construction chain is complete. The corresponding bounded-factor guarantee therefore does not apply.

\paragraph{Adaptive stochastic cover \citep{golovin2011}.} Cost-to-goal is a covering objective, making adaptive submodular cover the closest comparison. Its logarithmic approximation requires adaptive submodularity: expected marginal benefit, conditioned on observations to date, must be non-increasing. In the construction chain, the marginal contribution to confidence is zero for $b_1, \dots, b_d$ and becomes strictly positive at the unlock. Marginal benefit therefore increases along the chain, violating diminishing returns. The adaptive-submodular-cover guarantee does not apply; moreover, its approximation factor is bounded, whereas the capability-gating factor is unbounded.

\paragraph{Precedence-constrained stochastic probing and adaptivity gaps \citep{gns2017,bsz2019}.} These problems are structurally similar because they include precedence constraints and stochastic observations. Their principal comparison, however, is between adaptive and non-adaptive policies. The myopic-$d$ planner considered here is already adaptive because it replans after every observation, so it lies on the adaptive side of that comparison. Existing bounded adaptivity-gap results for submodular or XOS reward maximization under a budget do not characterize its cost relative to the adaptive optimum in this gated stochastic-covering problem.

\subsection{The dichotomy that forces a heuristic}\label{sec:dichotomy}

Increasing $d$ to the problem diameter does not provide a general solution. First, the lower bound has the quantifier structure $\forall d\,\exists\,\text{instance}$: for any fixed horizon, the construction chain can be made one step longer. Second, exact full-width lookahead grows exponentially with the horizon, and the required chain depth may be unknown. A practical alternative is a cost-to-go heuristic $h$ that estimates the remaining cost to reach $G$, including capabilities whose benefits lie beyond the current horizon, without explicitly expanding the full lookahead tree. Because such an estimate does not factor solely through within-horizon observations, it lies outside $\Pid$ and is not covered by Theorem~\ref{thm:unbounded}. Section~\ref{sec:explore-abl} measures the scaling difference on the testbed: full-width determinized lookahead grows by approximately a factor of $|A|$ per additional horizon step, whereas the relaxed-plan heuristic evaluates approximately $|A|$ candidates per decision. Randomized anytime search provides an intermediate alternative; \S\ref{sec:uct} evaluates it empirically and finds that it solves the witness with substantially greater per-decision search work that increases with chain depth and distractor count.

\subsection{Modeling assumptions}\label{sec:axioms}
\begin{enumerate}
\item Constructive experiments are uninformative about $q$ within the relevant horizon. This is a defining assumption of the clean separation, not an empirical claim. Mixed actions that both build a capability and provide partial information are outside the present result and require a new analysis of the probe-only argument.
\item In the regime of interest, the capability route is less costly than exhaustive direct probing. The unbounded statement concerns the instance family as $\gamma \to 0$; for any fixed $\gamma$, the ratio is finite, although it can be arbitrarily large across the family.
\item The precision probe is exact. Replacing it with a noisy probe changes the hard failure in Proposition~\ref{prop:fail} to a large but finite gap.
\item The result applies only to $\Pid$ and does not constrain planners with explicit capability value or sufficiently long-horizon planning. The class definition is therefore essential to the theorem.
\item The gate acts on the action space while the hypothesis language and query remain fixed. If constructive actions also expand the hypothesis language or query set, the same reachability principle may apply, but that richer setting requires a separate formalization.
\end{enumerate}

\section{CG-Plan: a cost-to-goal planner}\label{sec:cgplan}

The separation motivates a heuristic that accounts for capabilities beyond the lookahead horizon without requiring exhaustive search. We place this heuristic in a standard incremental-replanning loop. Capabilities remain action-space prerequisites: acquiring one changes which experiments can be reached. This is the setting covered by the theorem. Extending the same design to representational or hypothesis-generation capabilities would require enlarging the state; the implementation studied here plans only over an action-level capability dependency graph.

\paragraph{The heuristic as a relaxation.} We define $h(s; q)$ as the optimal cost of a delete-relaxation of the SSP, following classical planning. The relaxation assumes that constructive edges always succeed, each experiment returns its most decisive outcome, and anomalies do not interact. Under these assumptions, the stochastic AND/OR graph becomes a deterministic shortest-path problem over capability and information states. Because the relaxed problem is optimistic, its optimal cost lower-bounds the true expected cost to the goal. Admissibility follows from the relaxation rather than from a post hoc argument. In the chain-gated setting, the relaxed cost decomposes into capability and experimentation terms.

\paragraph{Capability term $\hcap$.} Reaching $G$ requires measuring the components of $\supp(q)$, some of which may require capabilities not contained in $I(s)$. On the deterministic capability graph defined by $\req$ and $\gain$, we compute the least-cost build subgraph that grants all missing capabilities from the current frontier, counting each constructive edge once:
\[ \hcap(s) = \min\ \text{cost of a build-DAG covering } \mathrm{req\_meas}(\supp(q)) \setminus I(s). \]
For a chain-structured graph, this has the closed form $\hcap(s) = c_b \cdot (\text{deepest required depth} - \text{current frontier depth})_+$. It is admissible because every policy that reaches $G$ must acquire the required capabilities, constructive edges are their only source, and all costs are positive. The term is computed by relaxed planning on the capability graph rather than by horizon expansion. The computational complexity depends on graph structure. For chains and tree-like graphs, the least-cost build subgraph can be obtained exactly by dynamic programming or shortest path. For general capability hypergraphs with multiple preconditions and grants, the problem becomes a minimum-cost directed Steiner problem and can be NP-hard. In that setting, CG-Plan uses a relaxed-plan approximation; admissibility is retained only when the approximation is a lower bound on the true construction cost.

\paragraph{Experimentation term $\hexp$.} After the required capabilities are assumed available, the posterior must still cross the confidence threshold. Using the SPRT lower bound from \S\ref{sec:separation}, we divide the remaining log-odds distance by the best information-per-cost rate among discriminating experiments that will become available:
\[ \hexp(s) = \big[\log\tfrac{1-\varepsilon}{\varepsilon} - |\mathrm{LLR}_s(q)|\big]_+ \big/ \max_e\big(\KL_e / c(e)\big). \]
In the binary-hypothesis witness, $\max_e(\KL_e/c(e))$ is a uniform upper bound on achievable information per unit cost over future beliefs. Dividing by this rate therefore underestimates the remaining experimentation cost and yields an admissible lower bound. For general adaptive design, an experiment may become more informative as the posterior changes, so this uniform-bound argument may fail. Outside the binary setting, $\hexp$ should therefore be interpreted as a practical heuristic rather than a guaranteed lower bound.

\paragraph{Relation to the lower bound.} The term $\hcap$ depends on the capability graph and the current set $I(s)$ rather than on a within-horizon observation distribution. It therefore assigns value to constructive actions whose downstream benefit is invisible to $\Pid$, placing $h$ outside the myopic class. At the same time, $\hcap$ is obtained from a relaxed shortest-path computation rather than from explicit deep lookahead. It thus provides a tractable estimate of beyond-horizon capability cost.

\paragraph{Path awareness.} The stated form of $\hcap$ is admissible only when the target is genuinely gated and construction is necessary. If the target can also be measured inexpensively without building, unconditional inclusion of the build cost overestimates the cost to the goal. The practical heuristic therefore compares two estimated routes---build then resolve, and resolve directly---and uses the less costly one. This preserves the construction option under gating while avoiding unnecessary builds in the leaky regime.

\paragraph{Additivity caveat.} The sum $\hcap + \hexp$ assumes a sequential cost model in which building and experimentation do not occur in parallel. If resources permit the two components to overlap in time, an admissible combination is $\max(\hcap, \hexp)$ rather than their sum. The guarantees below assume the sequential model.

\paragraph{The planner.} Stochastic interventional edges are determinized in the style of PPCP: the planner assumes the decisive outcome and replans when the observed outcome differs. Per-node bookkeeping follows LPA*/D*~Lite, with values $g$ and $rhs$ and key $k(u) = [\min(g, rhs) + h(u),\ \min(g, rhs)]$. Search is oriented backward from a virtual goal connected to every state satisfying $G$. This orientation accommodates the moving start state, which changes after every experiment. When edge costs change or new edges are proposed, only affected vertices are updated. At each step, the agent executes the first edge of the current plan, incorporates the outcome into the belief and capability set, and replans. CG-Plan therefore adapts established LPA*/D*~Lite incremental-replanning machinery \citep{koenig2004lpastar,koenig2002dstarlite} to belief space; the new elements are the discovery formulation and the capability-aware heuristic.

\paragraph{Inherited guarantees under idealized assumptions.} With deterministic outcomes, a realizable target, a complete proposer that supplies every edge needed by an optimal plan, and an admissible consistent heuristic, A*/LPA* optimality implies that the executed cost equals the optimal SSP cost. If $h$ is $(1+\varepsilon)$-admissible, the cost is at most $(1+\varepsilon)\OPT$. Under stochastic outcomes satisfying PPCP's clear-preference condition, PPCP's guarantees imply that the goal is reached with probability one. These assumptions need not hold in practice. The experiments therefore evaluate the practical effects of proposal completeness, heuristic error, determinization, and replanning separately.

\paragraph{Scope of the implemented planner.} We use a hand-coded heuristic to isolate the planning contribution from learned components, and a scripted complete proposer to isolate it from proposal quality. Natural extensions include learning $h$ from hindsight-labeled cost-to-go using capability-graph distance, log-odds gap, and realized information rate, and replacing the scripted proposer with a frozen language model. We defer these extensions. Extending $\hcap$ to state-dependent hypothesis or query languages would also require a different heuristic that estimates the cost of reaching an expressible representation rather than only a measurable component.

\section{Empirical illustration}\label{sec:experiments}

The experiments provide controlled tests of the separation mechanism. They evaluate whether information-based selectors prefer weak but immediately informative probes to zero-information builds, and whether CG-Plan constructs a capability when that route has lower estimated cost to the goal. These experiments do not establish that real scientific discovery problems are capability-gated; evaluating that claim would require rediscovery backtests of the kind discussed in \S\ref{sec:discussion}.

\subsection{Testbed and selectors}\label{sec:testbed}

\paragraph{Testbed.} The hidden system $M^*$ is a Boolean circuit with input, internal, and terminal nodes. Terminal nodes can be observed at low cost, but the internal subcircuit affects them only through a weak leakage path scaled by $\gamma$. At small $\gamma$, terminal probes therefore provide little information about the deeper structure. Internal nodes are organized into depth classes. Measuring a class-$j$ node requires tool $j$, and constructing tool $j$ requires tool $j-1$, yielding a capability chain of length $c$; each level also contains $N$ useless distractor builds. The query asks for the input--output \emph{function} of a deepest-class node, which is identifiable once the relevant node can be measured. We do not ask for exact wiring, which is not identifiable from the available probes. The agent maintains a belief ensemble over circuits and reweights it using the exact Gaussian likelihood. All methods use the same belief, archive, capability set, and experiment interface. Build outcomes are independent of the hidden circuit and therefore carry exactly zero information about the query, so capability indistinguishability holds numerically rather than only asymptotically. Because the capability graph is a chain, $\hcap$ has the exact closed form given in \S\ref{sec:cgplan} and is admissible in every reported run. The general-hypergraph approximation is not evaluated.

\paragraph{Methods and configuration.} Five selectors operate on the same harness, so their results differ only through action selection. \emph{Random-probe} samples uniformly from interventional probes and serves as a deliberately weak probing baseline. \emph{Plausibility-style} selects the probe that best separates the two leading hypotheses. \emph{Greedy EIG} maximizes one-step expected information gain about $q$ per unit cost, estimated by nested Monte Carlo. \emph{$H$-step EIG} is the bounded-horizon planner defined in \S\ref{sec:class}. \emph{CG-Plan} uses the path-aware form of $h = \hcap + \hexp$. Build actions are included in the candidate set for every value-based method---plausibility-style, greedy EIG, $H$-step EIG, and CG-Plan---so each method can construct a capability. Random-probe is the only method restricted to probes. Under capability indistinguishability, a build has exactly zero measured EIG, whereas a weak terminal probe has positive EIG; information-based rules therefore do not select the build. All methods use the same proposer: a fixed near-miss pool except in the dynamic-proposer experiment, where they share the same archive-consistent proposer. Ties in information scores are resolved in favor of the lower-cost action and never in favor of a zero-information build. The $H$-step planner computes determinized horizon-$H$ EIG exactly. It commits to construction only when the precision probe lies within the horizon; otherwise, it reduces to the greedy-EIG decision. The selected action is then simulated in the true circuit. Each experimental cell contains independently generated worlds, and results are reported as means with 95\% confidence intervals where applicable. Appendix~\ref{app:repro} gives sample counts and estimator-stability checks. Greedy EIG uses 300 outer samples per candidate by default and 80--120 in selected multi-world sweeps; all transition-band cells in Table~\ref{tab:gamma} use 300 samples. Rankings are stable in saturated regimes, whereas success within the transition band is sensitive to Monte Carlo error, as described in Result~2.

\subsection{Main separation}

\paragraph{Result 1: main separation.} In the gated regime ($\gamma = 0.02$, chain length $c=2$, $C_{\mathrm{build}} = 3$, and budget $30$), CG-Plan reaches the confidence target in all 100 independently generated worlds and reports the correct function in all 100. Its mean cost is $9.2 \pm 0.2$ (95\% CI), with $2.00 \pm 0.00$ builds and $3.2 \pm 0.2$ probes. Every myopic method completes zero builds, and none reaches the goal within budget: greedy EIG succeeds in 0/100 worlds, and the other information-based baselines succeed in 0/20 worlds in the original pass. Greedy EIG spends the full budget on probes. The construction chain itself costs $c\,C_{\mathrm{build}} = 6$, so CG-Plan's mean cost is the required build cost plus approximately three probes. Costs for the myopic methods are censored at the budget because they do not reach $G$. This finite-budget result matches the theorem's mechanism: as $\gamma$ decreases, direct resolution becomes increasingly expensive, whereas the build-then-measure route remains nearly constant. The number of builds before resolution directly reflects Lemma~\ref{lem:indist}: the myopic methods complete none, whereas CG-Plan completes exactly the required chain of length $c$.

\subsection{Gating controls: leakage and no-gating}

\paragraph{Result 2: dependence on the gating parameter.} We vary $\gamma$ while normalizing the target's leakage footprint so that no other quantity changes. CG-Plan reaches the goal in every world at approximately constant cost across the sweep. Greedy EIG succeeds in 0\% of worlds at small $\gamma$, transitions near $\gamma \approx 0.04$, and reaches 100\% at large $\gamma$. When greedy EIG succeeds in the visible regime, it is less costly than CG-Plan. Thus the separation occurs specifically when weak leakage makes the direct route expensive; it disappears when the target is readily observed. Saturated cells contain 8 worlds. The four transition cells ($\gamma = 0.03$--$0.06$) contain 32 worlds, use 300 outer EIG samples, and report Wilson 95\% intervals in Figure~\ref{fig:gamma} and Table~\ref{tab:gamma}. Across this band, greedy success increases monotonically from $0.12$ to $0.84$ and reaches $1.00$ by $\gamma = 0.08$, while CG-Plan remains at $1.00$. These estimates are consistent with the original 16-world pass. Within the transition band, candidate EIG values are close, so Monte Carlo variation affects the selected probe and the measured success rate. A separate 32-world run with 80 outer samples produced success rates $(0.03, 0.50, 0.84, 1.00)$ across $\gamma = 0.03$--$0.06$. We therefore report the estimator setting explicitly and interpret the exact transition-band rates as estimator-dependent. The location and monotone direction of the transition are stable across runs.

\begin{figure}[h]
\centering
\includegraphics[width=0.72\linewidth]{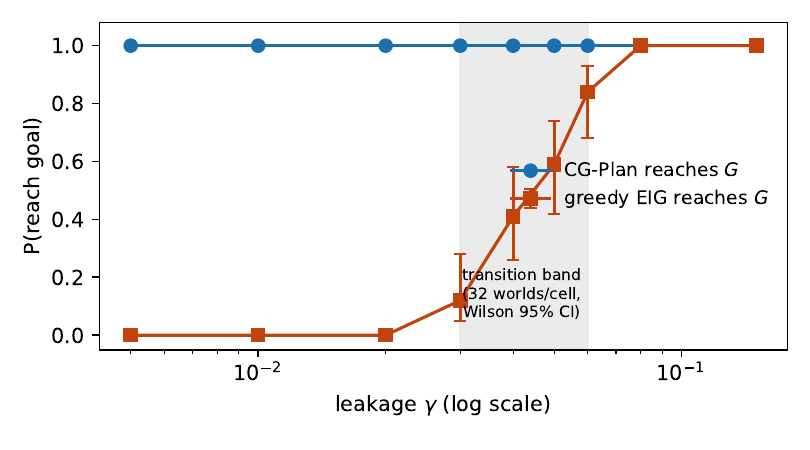}
\caption{Success as a function of leakage $\gamma$. Greedy EIG transitions from $0$ to $1$ across a narrow range (shaded; 32 worlds per transition cell, with Wilson 95\% intervals), whereas CG-Plan remains at $1.00$. Table~\ref{tab:gamma} reports the numerical values.}
\label{fig:gamma}
\end{figure}

\begin{table}[h]
\centering
\small
\caption{Regime sweep. Greedy EIG succeeds only when leakage is sufficiently strong, whereas CG-Plan reaches the goal throughout. Transition cells ($\gamma=0.03$--$0.06$) contain 32 worlds, use $n_{\mathrm{outer}}=300$, and report Wilson 95\% intervals. Saturated cells contain 8 worlds. A dash denotes a quantity not logged in that pass. CG-Plan's reported answer was correct in every logged run, including all 100 runs in Result~1.}
\label{tab:gamma}
\begin{tabular}{@{}lccc@{}}
\toprule
$\gamma$ & greedy reaches $G$ & CG-Plan reaches $G$ & CG-Plan correct \\
\midrule
0.005 & 0.00 & 1.00 & 1.00 \\
0.010 & 0.00 & 1.00 & 1.00 \\
0.020 & 0.00 & 1.00 & 1.00 \\
0.030 & 0.12 [0.05, 0.28] & 1.00 & -- \\
0.040 & 0.41 [0.26, 0.58] & 1.00 & -- \\
0.050 & 0.59 [0.42, 0.74] & 1.00 & -- \\
0.060 & 0.84 [0.68, 0.93] & 1.00 & -- \\
0.080 & 1.00 & 1.00 & 1.00 \\
0.150 & 1.00 & 1.00 & 1.00 \\
\bottomrule
\end{tabular}
\end{table}

\paragraph{No-gating control.} With a strong terminal footprint, corresponding to large $\gamma$, the target is inexpensive to observe directly. Greedy EIG reaches the goal without construction at cost approximately $3$, and path-aware CG-Plan selects the same direct route. CG-Plan is therefore not constrained to build; its decision depends on which route has lower estimated cost to the goal. The performance gap is present under gating and absent in this control.

\subsection{Every finite horizon is defeated}

\paragraph{Result 3: finite-horizon boundary.} We vary the lookahead horizon $H$ and the world's construction-chain length $c$. The $H$-step planner reaches the goal if and only if $c < H$, producing the upper-triangular boundary in Table~\ref{tab:quantifier}. CG-Plan, whose capability term does not depend on a lookahead horizon, solves every chain length. Here $H$ includes the precision probe as one lookahead step, so a chain of length $c$ becomes visible when $c+1 \le H$, equivalently $c < H$. This matches the theorem: a chain of length $d+1$ is invisible to horizon $d$. The experiment therefore realizes the quantifier structure $\forall d\,\exists\,\text{instance}$. The bounded-horizon commitment decision is computed on the determinized model, and its consequences are simulated in the true circuit. When the planner selects the chain, the belief resolves; when it continues direct probing, it does not reach the goal.

\begin{table}[h]
\centering
\small
\caption{Quantifier grid: goal reached (1) or not (0). Five worlds per cell; every cell is unanimous. The $H$-step planner succeeds iff $c<H$; CG-Plan succeeds for every chain length.}
\label{tab:quantifier}
\begin{tabular}{@{}lcccc@{}}
\toprule
 & $c=1$ & $c=2$ & $c=3$ & $c=4$ \\
\midrule
$H=1$ & 0 & 0 & 0 & 0 \\
$H=2$ & 1 & 0 & 0 & 0 \\
$H=3$ & 1 & 1 & 0 & 0 \\
$H=4$ & 1 & 1 & 1 & 0 \\
$H=5$ & 1 & 1 & 1 & 1 \\
CG-Plan & 1 & 1 & 1 & 1 \\
\bottomrule
\end{tabular}
\end{table}

\subsection{Dynamic proposer and near-miss hypotheses}

\paragraph{Result 4: near-miss hypotheses under dynamic proposal.} The separation requires live hypotheses that agree on inexpensive observations but differ on the gated target. To test whether this structure is an artifact of a fixed hypothesis pool, we replace the fixed proposer with a dynamic proposer that returns hypotheses consistent with the current archive. At small $\gamma$, terminal observations weakly constrain the deep target, so archive-consistent hypotheses continue to agree at the terminals while differing on the target. In the canonical world, the proposer returns 22 candidate target functions with an empty archive, 21 after 16 terminal probes, and 1 after 6 gated internal probes. Across worlds, it returns 22--35 candidates initially; terminal data removes at most one, whereas gated data reduces the set to 1--2. Thus the proposer's uncertainty about $q$ persists until it receives gated observations. With fully dynamic proposal, the separation occurs in all 20 worlds and CG-Plan reports the correct answer in all 20. These results show that the required near-miss structure also arises from archive-consistent proposal under gating.

Result~4 also motivates a broader reachability question. In the present experiments, the hypothesis language is fixed and gating determines which measurements constrain the live hypotheses. In open-ended discovery, a constructive action may instead make a new representation, primitive, or hypothesis class expressible. We treat that setting as a possible extension rather than as a result of this paper.

\subsection{Robustness and scaling}

\paragraph{Result 5: distractor robustness.} Adding \emph{distractor} builds that grant capabilities irrelevant to $q$ does not change CG-Plan's behavior. For $N_{\mathrm{distractors}} \in \{0,2,4,8\}$, its cost remains $9.2 \pm 0.3$ (20 worlds per cell), and it completes exactly two builds in every world. Greedy EIG fails in all 12 worlds at every value of $N$. The capability term explains this behavior: $\hcap$ computes the least-cost build subgraph that covers the capabilities required by $q$, so off-path capabilities do not enter the estimate.

\paragraph{Scaling in the chain testbed.} In this controlled family, the number of executed planner steps grows approximately linearly with chain depth. For $c \in \{1,2,3,4\}$, the mean step counts are $4.8$, $5.2$, $6.4$, and $8.0$, respectively, and the goal is reached in every run. These counts consist of approximately $c$ builds plus a small, nearly constant number of probes. This result is specific to the chain testbed, where $\hcap$ has a closed form and candidate evaluation is linear in the action set. General capability DAGs may require substantially more expensive relaxed planning.

\subsection{Exploration, ablation, and the price of lookahead}\label{sec:explore-abl}

\paragraph{Result 6: undirected exploration under distractors.} This experiment evaluates Proposition~\ref{prop:explore}. An $\varepsilon$-greedy selector with $\varepsilon_0 = 0.1$ chooses uniformly among available actions during exploration and otherwise follows the plausibility selector, which belongs to $\Pid$. We vary the distractor count $N$ under a budget of $300$, ten times the standard budget, so successful runs can be measured rather than immediately censored (Table~\ref{tab:explore}). When the direct-probe route remains open ($\gamma = 0.02$), exploration reaches the goal in every world through probing, but at 11--19 times the cost of CG-Plan. Its mean number of builds rises from $0.6$ to $13.1$ as $N$ increases from $0$ to $64$, with most additional builds spent on distractors. When the probe route is effectively closed ($\gamma = 0.005$), exploration fails in all runs: 0/10 at $N=0$, where it completes only one of the two required builds on average, and 0/10 at $N \in \{16,64\}$, where it completes 15 and 24 builds, respectively, mostly distractors. These builds consume 45 and 72 units of the 300-unit budget. CG-Plan reaches the goal in every world at cost approximately $9.4$ with exactly two builds. In this testbed, a distractor is removed after it is built, so the action set shrinks over time. This makes exploration easier than in the repeatable-distractor instance of \S\ref{sec:explore}; failure occurs despite that advantage.

\begin{table}[h]
\centering
\small
\caption{$\varepsilon$-greedy exploration ($\varepsilon_0 = 0.1$) vs distractor count $N$, budget 300, ten worlds per cell. With the probe route open exploration succeeds only expensively via probing; with it closed, exploration is diluted by decoys and fails.}
\label{tab:explore}
\begin{tabular}{@{}llccccc@{}}
\toprule
$\gamma$ & $N$ & \multicolumn{3}{c}{$\varepsilon$-greedy} & \multicolumn{2}{c}{CG-Plan} \\
\cmidrule(lr){3-5} \cmidrule(lr){6-7}
 & & goal & cost & builds & goal & cost \\
\midrule
0.02 (probe route open) & 0 & 1.00 & 139.2 & 0.6 & 1.00 & 9.4 \\
 & 4 & 1.00 & 99.0 & 1.9 & 1.00 & 9.4 \\
 & 16 & 1.00 & 134.6 & 5.8 & 1.00 & 9.4 \\
 & 64 & 1.00 & 177.9 & 13.1 & 1.00 & 9.4 \\
0.005 (probe route closed) & 0 & 0.00 & 300.0 & 1.0 & 1.00 & 9.4 \\
 & 16 & 0.00 & 300.0 & 15.0 & 1.00 & 9.4 \\
 & 64 & 0.00 & 300.0 & 24.0 & 1.00 & 9.4 \\
\bottomrule
\end{tabular}
\end{table}

\paragraph{Result 7: heuristic ablation.} We evaluate four modified heuristics in both a gated and a no-gating regime (Table~\ref{tab:ablation}). The capability term $\hcap$ is responsible for selecting construction: the experiment-only variant never builds and fails in every gated world, matching the greedy-EIG baseline. The experimentation term $\hexp$ is responsible for selecting informative probes after construction: the capability-only variant completes the correct chain in every gated world but then stalls because all post-build successor states have the same zero heuristic, causing the one-step argmin to select inexpensive but uninformative probes. The path-aware comparison prevents unnecessary construction. The core variant $\hcap + \hexp$, without comparison to the direct route, solves the gated cell but completes two unnecessary builds in the no-gating cell, increasing cost from $4.1$ to $9.2$. The $h=0$ control fails in every gated world and in 19 of 20 no-gating worlds. This final row also clarifies the implementation: the evaluated planner uses the one-step rule $\arg\min_a\, c(a) + h(\mathrm{succ}(a))$ on the relaxed model, so all directional guidance comes from $h$. Separating the graph formulation from the heuristic would require evaluating the full incremental-search implementation; the $h=0$ row cannot make that distinction.

\begin{table}[h]
\centering
\small
\caption{Heuristic ablation, 20 worlds per cell, budget 30, $c = 2$, $N = 2$. Goal / mean cost / mean builds.}
\label{tab:ablation}
\begin{tabular}{@{}lcccccc@{}}
\toprule
 & \multicolumn{3}{c}{gated ($\gamma = 0.02$)} & \multicolumn{3}{c}{no-gating ($\gamma = 0.15$)} \\
\cmidrule(lr){2-4} \cmidrule(lr){5-7}
variant & goal & cost & builds & goal & cost & builds \\
\midrule
full (path-aware) & 20/20 & 9.2 & 2.0 & 20/20 & 4.1 & 0.0 \\
core ($\hcap + \hexp$) & 20/20 & 9.2 & 2.0 & 20/20 & 9.2 & 2.0 \\
cap-only ($\hcap$) & 0/20 & 30.0 & 2.0 & 1/20 & 28.9 & 2.0 \\
exp-only ($\hexp$) & 0/20 & 30.0 & 0.0 & 20/20 & 4.1 & 0.0 \\
$h = 0$ & 0/20 & 30.0 & 0.0 & 1/20 & 28.6 & 0.0 \\
\bottomrule
\end{tabular}
\end{table}

\paragraph{Measured cost of full-width lookahead.} On the Result~1 action set, which has $|A| = 11$ actions ($8$ terminal probes and $3$ builds), full-width determinized search expands $135$, $1{,}571$, and $18{,}661$ nodes per decision at horizons $H=2$, $3$, and $4$. The measured growth factor is approximately $|A|$ per additional horizon step, consistent with $|A|^H$, whereas CG-Plan evaluates approximately $|A|$ candidates per decision at every chain length. The absolute enumeration cost remains modest in this deliberately small testbed: even $H=5$ requires only about $1.6 \times 10^5$ nodes. The experiment therefore demonstrates the scaling rate rather than a prohibitive runtime at this action-set size. With hundreds of candidate experiments, the same $|A|^H$ dependence becomes substantially more costly while the relaxed-plan heuristic remains linear in $|A|$. The $H$-step baseline used for the success results in Table~\ref{tab:quantifier} employs an oracle-efficient shortcut that makes the same construction decision as full-width determinized search without enumerating the tree. The success results are therefore unchanged; the node counts reported here come from genuine full-width enumeration.

\subsection{Generic rollout search}\label{sec:uct}

Randomized anytime search lies between bounded-horizon information selection and exhaustive lookahead, so it is not covered by either part of the preceding comparison. We evaluate a single-player UCT planner in the same determinized belief model used to derive CG-Plan's heuristic. The planner receives the identical capability graph, deterministic build dynamics, decisive-outcome KL values recomputed at the current belief, and the true set of actions available at the root; it replans after every executed action. Rollouts are uniform random and limited to 18 steps, with average backup, goal reward 100, and exploration constant 30. This favorable configuration isolates the difference between generic search and the capability-aware heuristic under shared dynamics. Evaluation criteria were specified before the runs: a cell is considered solved when at least 9 of 10 worlds reach the goal; the crossover is the smallest simulation budget meeting this threshold; and computation is measured in model-step operations per decision, excluding the approximately 60 KL belief evaluations shared by both methods. Wall-clock time is also reported, although the shared KL computation dominates at low simulation budgets.

\paragraph{Result 8: rollout-search cost under gating.} Table~\ref{tab:uct} reports the simulation-budget sweeps. In the standard cell, UCT first solves the task at 300 simulations per decision, corresponding to $2{,}643$ model steps. In the hard cell, the crossover is $1{,}000$ simulations and $11{,}471$ model steps. At chain depth $5$, the crossover rises to $3{,}000$ simulations and $30{,}486$ model steps; budgets of 300 and $1{,}000$ solve only 6/10 and 7/10 worlds. For $c \in \{1,\dots,4\}$ with $N=2$, the crossover budget remains 300, but per-decision work increases from $2{,}279$ to $4{,}459$ model steps before rising by an order of magnitude at $c=5$. Distractors increase the required budget independently: at $c=4$, the crossover is 300 simulations for $N=2$ and $1{,}000$ for $N=8$. By comparison, CG-Plan uses approximately 11--19 heuristic evaluations per decision at every depth. At $c=5$, it reaches the goal in all 10 worlds at cost $18.4 \pm 0.7$, completes exactly five builds, and requires 184 ms per step. By the reported operation counts, search at the crossover uses two to three orders of magnitude more per-decision work, with the difference increasing along both gating dimensions. Below the crossover, performance changes abruptly. At 100 simulations in the hard cell, UCT completes the full four-build chain in every world but reaches the goal in none, paralleling the capability-only ablation. At 300 simulations, it completes 5.20 builds on average despite requiring only four, indicating distractor construction. In the no-gating control, UCT reaches the goal in all 10 worlds at cost $5.6$ but averages 0.40 unnecessary builds per world, accounting for most of its cost premium over path-aware CG-Plan (cost $4.6$, zero builds).

\paragraph{Interpretation.} At $10{,}000$ simulations per decision, UCT closely matches CG-Plan in the hard cell: costs are $15.4$ and $15.3$, respectively, and both complete exactly four required builds. Thus the computational claim in \S\ref{sec:dichotomy} is an exponential-cost statement about exhaustive lookahead, not an impossibility result for sampled search. On this witness, generic rollout search recovers the capability route with sufficient computation. The empirical distinction has three parts: the required search work grows with chain depth and distractor count; performance exhibits a budget threshold below which the route may be constructed but not exploited; and sampled search does not provide the explicit cost lower bound or deterministic build restraint supplied by the heuristic. UCT also succeeds only because its transition model includes the capability graph and construction dynamics, which are the same structures used by $\hcap$. The capability-aware model is therefore necessary for both approaches to escape Theorem~\ref{thm:unbounded}. We did not evaluate a raw belief-space POMCP solver without these structures, so the comparison is limited to search versus heuristic under a shared determinized model.

\begin{table}[h]
\centering
\small
\caption{UCT on the determinized belief MDP compared with CG-Plan. Each row contains 10 worlds with budget 30; a cell is solved when at least 9/10 worlds reach the goal. Model evaluations are model-step operations per decision and exclude the approximately 60 shared KL evaluations. The crossover budgets are 300 simulations in the standard cell and $1{,}000$ in the hard cell; the $c=5$, $N=2$ scaling run crosses at $3{,}000$.}
\label{tab:uct}
\begin{tabular}{@{}llccccc@{}}
\toprule
cell & sims/step & goal & cost & builds & model-evals/step & ms/step \\
\midrule
std ($c{=}2, N{=}2$) & 30 & 0.00 & 30.0 & 1.1 & 469 & 554 \\
 & 100 & 0.00 & 30.0 & 2.1 & 937 & 707 \\
 & 300 & \textbf{1.00} & 10.1 & 2.0 & 2{,}643 & 642 \\
 & 1{,}000 & 1.00 & 10.3 & 2.0 & 5{,}027 & 663 \\
 & CG-Plan & 1.00 & 9.4 & 2.0 & -- & 408 \\
\midrule
hard ($c{=}4, N{=}8$) & 100 & 0.00 & 30.1 & 4.0 & 1{,}897 & 513 \\
 & 300 & 0.70 & 23.6 & 5.2 & 5{,}088 & 586 \\
 & 1{,}000 & \textbf{1.00} & 17.5 & 4.3 & 11{,}471 & 694 \\
 & 10{,}000 & 1.00 & 15.4 & 4.0 & 39{,}024 & 1{,}128 \\
 & CG-Plan & 1.00 & 15.3 & 4.0 & -- & 402 \\
\bottomrule
\end{tabular}
\end{table}

\subsection{Auditing CG-Plan}\label{sec:audit}

Because $h$ estimates remaining cost to the goal, CG-Plan can report an ex ante cost estimate in addition to selecting an action.

\paragraph{Initial cost estimate.} In the Result~1 cell, $h(s_0) = 6.10$ in every world because the estimate depends on the capability graph and achievable information rates rather than on the sampled world. The realized mean cost is $9.25 \pm 0.34$ over 20 worlds. The estimate is approximately two-thirds of realized expenditure; most of the difference is about three probe-cost units introduced by the decisive-outcome relaxation, which represents evidence accumulation fractionally. In these runs, the estimate provides a useful lower-cost benchmark before execution. Greedy EIG, $\varepsilon$-greedy selection, and rollout search do not produce an analogous explicit estimate.

\paragraph{Empirical admissibility.} Across the 105 states visited in these episodes, $h(s)$ never exceeds the realized remaining cost; no violations are observed (Figure~\ref{fig:audit}, left). This property is not guaranteed for the testbed. Section~\ref{sec:cgplan} establishes admissibility of $\hexp$ only for the binary witness, whereas the testbed query has multiple possible function values. We therefore report the result as an empirical observation rather than a theorem.

\paragraph{Allocation of cost.} We classify each unit of expenditure as a required build, distractor build, informative probe (per-step KL about $q$ at least $\tau = 0.1$ nats), or near-zero-information probe (Figure~\ref{fig:audit}, right). CG-Plan allocates $64.9\%$ of cost to required builds and $35.1\%$ to informative probes, with no expenditure on distractors or near-zero-information probes. Greedy EIG allocates all expenditure to near-zero-information probes in the same cell. In the closed-route $\varepsilon$-greedy cell of Table~\ref{tab:explore}, 85\% of expenditure goes to near-zero-information probes, 14\% to distractor builds, and 1\% to required builds. These allocations illustrate the effect of pricing capabilities with respect to the target query.

\begin{figure}[h]
\centering
\includegraphics[height=0.30\linewidth]{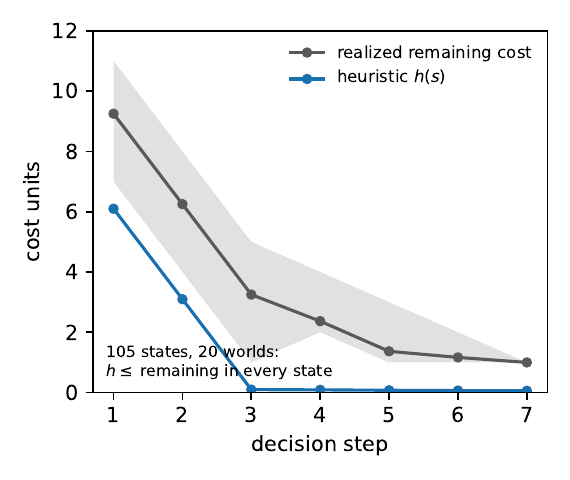}\hfill
\includegraphics[height=0.30\linewidth]{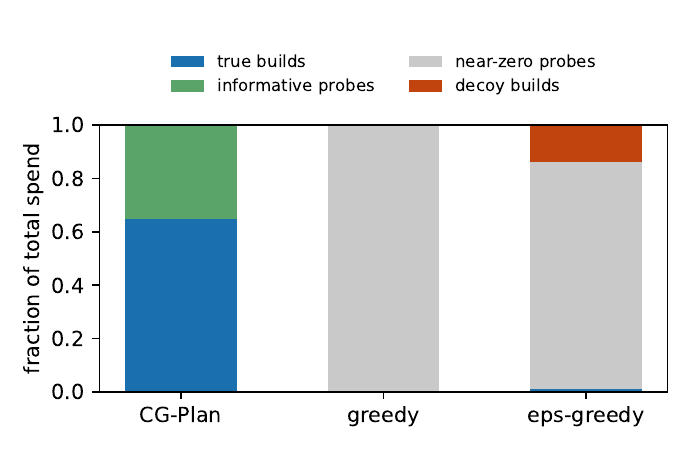}
\caption{Audit of CG-Plan. \emph{Left:} realized remaining cost to the goal (upper curve) and heuristic estimate $h(s)$ (lower curve) at each decision; bands show the minimum and maximum across 20 Result~1 worlds. The heuristic remains below realized remaining cost in all 105 visited states. \emph{Right:} allocation of total expenditure by action type. CG-Plan uses no distractor builds or near-zero-information probes; the $\varepsilon$-greedy distribution is from the closed-route cell.}
\label{fig:audit}
\end{figure}

\subsection{Limitations visible in the runs}

Two limitations are evident in the experiments. First, in the transition range $\gamma \approx 0.04$--$0.08$, CG-Plan sometimes estimates the direct route as slightly more costly than the build route and constructs a chain that a fully informed planner would omit. This overconstruction reflects approximation error in the relaxed cost-to-go estimate and motivates the path-aware route comparison. Second, all empirical results are controlled demonstrations on constructed instances. They establish the proposed mechanism and the planner's behavior on those instances, not the prevalence of capability gating in real discovery problems (\S\ref{sec:discussion}).

\section{Related work}\label{sec:related}

\paragraph{Sequential experimental design and the value of information.} The analyzed selection rule originates in Bayesian optimal experimental design, which chooses the experiment with the greatest expected information gain about a target \citep{lindley1956,mackay1992,chaloner1995}. Modern methods use amortized or variational EIG estimators and policies trained for adaptive design \citep{rainforth2018,foster2019,foster2021}. Some recent discovery loops pair a frozen language-model proposer with an outer information-based selector; one example is an agentic causal Bayesian optimization loop that maximizes pairwise information gain \citep{roy_acbo}. Selection rules of this form belong to $\Pid$ because they value an experiment through the outcome distribution it induces. Information gain remains appropriate in the no-gating special case described in \S\ref{sec:ssp}; the limitation identified here is its inability to represent capability creation when that creation determines the least-cost route to the target. Section~\ref{sec:irreducibility} distinguishes the present horizon-indexed, unbounded gap from bounded value-of-information results on the curvature axis \citep{krauseguestrin2005,krauseguestrin2009}.

\paragraph{Costly information acquisition with stopping: Pandora's Box.} Weitzman's Pandora's Box model is the economic formulation most closely related to our objective \citep{weitzman1979}. An agent opens costly boxes, observes their contents, and stops when the best observed option exceeds the value of further search. Weitzman's index policy orders boxes by independently computed reservation values. The key difference is action availability: the classical model assumes that every box can be opened from the start, in addition to assuming known reward distributions and independence. Subsequent work relaxes known-prior and independence assumptions \citep{gergatsouli2022,gergatsouli2023} and introduces opening-order constraints \citep{boodaghians2020}. These models still do not include an action whose effect is to create or enable a previously unavailable box. Capability acquisition provides precisely that operation. The belief-space shortest-path formulation retains costly stopping while representing changes in future action availability.

\paragraph{Heuristic search and incremental replanning.} CG-Plan combines established planning mechanisms. A* provides the admissible-heuristic cost-to-go framework \citep{hart1968}, and weighted A* provides a bounded-suboptimal variant \citep{pohl1970}. Lifelong Planning A* and D*~Lite reuse search results as the start state moves and edge costs change \citep{koenig2004lpastar,koenig2002dstarlite}, building on D* \citep{stentz1994}. D*~Lite is appropriate here because each executed experiment moves the current epistemic state, while newly proposed experiments and revised costs modify edges. Stochastic belief transitions are determinized in the style of PPCP \citep{likhachev2009ppcp}; the underlying SSP also relates to LAO* \citep{hansen2001lao} and RTDP \citep{barto1995rtdp}. The heuristic is the optimal cost of a delete-relaxation, following relaxed-plan heuristics in classical planning \citep{bonet2001,hoffmann2001ff}. One-time capability enablement is also loosely analogous to preprocessing shortcuts in contraction hierarchies \citep{geisberger2008}, although the mechanism here changes action availability rather than contracting graph nodes. The search algorithms are not new; the contribution is their application to capability-gated discovery and the capability term in the heuristic.

\paragraph{Submodularity, adaptive submodularity, and stochastic probing.} The closest lower-bound literature concerns greedy maximization of submodular objectives \citep{nwf1978}, adaptive submodularity \citep{golovin2011}, and stochastic probing with precedence constraints \citep{gns2017,bsz2019}. Adaptive submodular cover is the most direct comparison because the present objective is minimum cost to a confidence target. Its bounded guarantees require a fixed selectable ground set and diminishing marginal returns. The gated chain violates these conditions: the decisive measurement is initially unavailable, and the marginal benefit of successive builds increases from zero to a positive value at the unlock. Section~\ref{sec:irreducibility} states the resulting distinctions in detail.

\paragraph{Frozen-policy, growing-artifact systems.} CG-Plan is intended for an architecture in which the proposer remains fixed while capabilities and knowledge accumulate in an external versioned store. Related systems include Voyager's skill library for open-ended control \citep{wang2023voyager}, FunSearch's evolved program database \citep{romeraparedes2024funsearch}, AlphaEvolve's evolutionary code search \citep{novikov2025alphaevolve}, and DreamCoder's learned library of reusable abstractions \citep{ellis2021dreamcoder}. These systems show that a fixed generator can improve over time through a growing external artifact, and that a capability may be a skill, program fragment, abstraction, representation, or other reusable object rather than only a physical instrument. Their selection objectives differ from ours: they expand their libraries through greedy or evolutionary task-performance search rather than by minimizing cost to a confidence target over a capability-gated graph.

\paragraph{Intrinsic objectives related to reachability.} Three lines of work can encourage capability acquisition without explicitly minimizing cost to the query. \emph{Empowerment} maximizes channel capacity from action sequences to future states \citep{klyubin2005,salge2014}. Because a capability enlarges the feasible action set, empowerment may favor construction, but it is query-independent. On the distractor family in \S\ref{sec:explore}, a decoy capability can increase empowerment as much as the required chain, and no cost-to-go guarantee follows. \emph{Options and skill discovery} represent capabilities as temporally extended actions \citep{sutton1999options,eysenbach2018diayn}. Their values are reward-based; if the relevant reward lies beyond a bounded evaluation horizon, the option can exhibit the same visibility problem. \emph{Curiosity and novelty bonuses} depend on their state representation. Prediction-error curiosity over observations \citep{pathak2017} is an information functional, so a predictable build outcome receives little value. Count-based novelty over capability states does reward construction, but without query direction it can spend $\Theta(Nc_b)$ on distractors. These objectives may reach gated capabilities, but they do not price them by their contribution to the least-cost route for $q$. General belief-space planners, including POMCP \citep{silver2010pomcp}, DESPOT \citep{somani2013despot}, and belief-space task planning \citep{kaelbling2013}, lie outside $\Pid$ and can discover construction chains through search. Sections~\ref{sec:explore-abl} and~\ref{sec:uct} compare full-width and rollout search with the relaxed-plan heuristic on the testbed.

\paragraph{Automated scientific discovery systems.} End-to-end discovery systems include the AI Scientist and its tree-search successor \citep{lu2024aiscientist,yamada2025aiscientist2}, Google's AI co-scientist \citep{gottweis2025coscientist}, FutureHouse's Kosmos and Robin \citep{mitchener2025kosmos,ghareeb2025robin}, and platforms such as Ai2's Asta \citep{allenai_asta}. These systems motivate the present analysis. Many now search over ideas or experiments rather than applying only one-step ranking, but candidate scores are commonly based on predicted quality, novelty, plausibility, or information rather than minimum cost to a confidence target under gated action availability. The separation identifies a regime in which such scores can be inefficient: the least-cost route contains a capability chain whose benefit is not represented in the score. When an LLM assigns the score directly, membership in $\Pid$ is not automatic (\S\ref{sec:class}); the theorem applies only when the score approximates a bounded-horizon observation functional. Whether model-based scorers independently assign value to capability acquisition is an empirical question. The builds-before-resolution diagnostic in \S\ref{sec:experiments} can be applied to any selector as a black-box test.

\paragraph{Rediscovery backtesting.} Whether real scientific problems contain deep capability chains is an empirical question suited to rediscovery backtesting, which asks whether a method recovers a known result using only period-appropriate evidence \citep[e.g.,][for materials-science relationships recoverable from earlier literature]{tshitoyan2019}. A historical complement to the present lower bound would test whether reaching a known result required a capability chain that a myopic selector would not have constructed. In a broader formulation, the relevant capability might be an instrument, measurement procedure, representation, abstraction, or hypothesis class that was not initially available.

\section{Discussion: scope and limitations}\label{sec:discussion}

\paragraph{Lower-bound witnesses.} The instance family and testbed are adversarial witnesses, as is appropriate for a lower bound. They show that capability gating can defeat myopic selection and that CG-Plan handles the constructed mechanism. They do not establish that real scientific discovery problems commonly contain deep capability chains. That claim requires separate empirical evidence, such as historical rediscovery studies using period-appropriate information.

\paragraph{Contributions and inherited components.} Once $\Pid$ is defined, the separation proof is direct. Its role is to identify reachability as a distinct axis, establish the limitation of the myopic class, and motivate a capability-aware heuristic. The LPA*/D*~Lite replanning machinery is standard. The principal contributions are the cost-to-goal formulation of discovery with graph-changing constructive actions, the separation of capability reachability from curvature and information order, and the capability term that estimates beyond-horizon construction cost without explicit deep lookahead.

\paragraph{Known capability graph.} CG-Plan receives the capability dependency graph---the $\req$ and $\gain$ sets for every action---as input, and $\hcap$ uses this graph directly. In real discovery, the gating structure may itself be uncertain: the agent may not know which instrument or procedure will make a quantity measurable. If the graph is misspecified, $\hcap$ can lose admissibility in either direction. Spurious gates inflate the estimate, whereas omitted gates make it too small. Learning action preconditions and capability effects from experience is therefore a major prerequisite for deployment beyond constructed testbeds.

\paragraph{Measurement gates and epistemic gates.} The formal separation concerns gated action spaces in which constructive actions unlock downstream experiments. This is the simplest setting in which the mechanism can be isolated. Open-ended discovery may also contain \emph{epistemic} gates: actions that make a representation, primitive, hypothesis class, or query type available. A model of that setting would include state-dependent hypothesis and query languages, for example $s=(B,D,I,\mathcal{L},\mathcal{Q},v)$. A constructive action could then change what the agent can express as well as what it can measure. We do not prove results for this richer model. Result~4 indicates why it may matter: even with a fixed language, the proposer remains uncertain until gated data are observed; with an evolving language, the decisive hypothesis may itself be unreachable before a representational construction.

\paragraph{Open directions.} Several extensions remain. The heuristic could be learned from hindsight cost-to-go while preserving explicit capability-graph features. The scripted proposer could be replaced by a frozen language model to test whether archive-consistent near-miss hypotheses are generated reliably. Mixed actions that both construct and inform would require a new lower-bound analysis. Practical systems also require multi-resource objectives in which time, money, computation, and risk are constrained separately. State-dependent hypothesis and query languages would extend action gating to full epistemic reachability. Finally, historical rediscovery studies are needed to determine whether capability gating is common in scientific practice.

\appendix

\section{Full proofs}\label{app:proofs}

For completeness, we restate the instance family and planner class before proving Lemma~\ref{lem:indist}, Theorem~\ref{thm:unbounded}, and Proposition~\ref{prop:fail}.

\paragraph{Instance family (restated).} Fix $d \ge 1$. Instance $I(d,\gamma,c_b,\varepsilon)$ has a hidden bit $q\in\{0,1\}$ with a uniform prior, belief $p=\Pr[q=1]$, goal $G=\{\max(p,1-p)\ge 1-\varepsilon\}$, and a minimum-cost-to-goal objective. The direct probe $\Bdir$ has cost $1$, is always available, and returns $o\in\{0,1\}$ with $\Pr[o=q]=(1+\gamma)/2$. Its per-call divergence is $\KL=\gamma\ln\frac{1+\gamma}{1-\gamma}=\Theta(\gamma^2)$. The chain $b_1,\dots,b_{d+1}$ consists of actions with cost $c_b>0$, with $b_i$ requiring $b_{i-1}$; each action returns a $q$-independent observation. The precision probe $\Bprec$ has cost $1$, becomes available after $b_{d+1}$, and returns $q$ exactly. The class $\Pid$ selects $\pi(s)\in\arg\max_a F_s(c(a),O^{\le d}_{s,a})$, where $F_s$ is monotone in information about $q$ per unit cost and depends on an action only through its cost and within-$d$-step observation ensemble.

\begin{proof}[Proof of Lemma~\ref{lem:indist}]
The score assigned to $b_i$ depends only on its cost and on the joint distribution of observations reachable within $d$ steps. Advancing the construction chain changes only the availability of $\Bprec$, which yields an informative observation after the remaining $r(s)$ builds and one additional probe action. Thus its observation occurs at depth $r(s)+1$ from the current state. If $r(s)>d$, no sequence of at most $d$ actions beginning with $b_i$ reaches that observation. Within the horizon, $b_i$ and $\oslash_{c_b}$ therefore have the same cost, both return a $q$-independent observation immediately, and both leave every informative observation law unchanged. Condition~(i) of Definition~\ref{def:pid} implies that every admissible $F_s$ assigns them equal value. The direct probe $\Bdir$ has positive divergence $\Theta(\gamma^2)$, whereas the build and null actions carry zero information. Condition~(iii) therefore ranks $\Bdir$ strictly above $b_i$ for any $c_b>0$; no comparison between $c_b$ and the probe cost is required. At the initial state, $r(s_0)=d+1>d$, so no policy in $\Pid$ selects $b_1$ while $\Bdir$ is available.
\end{proof}

\noindent\emph{Why the first-build statement suffices.} The argument does not require every build to remain invisible in every state. After $k$ builds, $r(s)=d+1-k$, so the precision-probe observation lies at depth $d+2-k$ and enters the $d$-step horizon once $k\ge 2$. However, a policy in $\Pid$ strictly prefers $\Bdir$ to $b_1$ at $s_0$. It never advances the chain, never reaches a state with $r(s)\le d$, and remains confined to $\Bdir$ by induction.

\begin{proof}[Proof of Theorem~\ref{thm:unbounded}]
\emph{Optimal cost.} Building the chain and querying $\Bprec$ gives $\OPT \le (d+1)c_b + 1$ and reaches confidence $1\ge 1-\varepsilon$ deterministically. A partial-build-then-probe strategy cannot improve this bound: before the chain is complete, the builds provide no information and do not make $\Bprec$ available, so such a strategy is a direct-probe strategy with additional construction cost.

\emph{Myopic cost.} By Lemma~\ref{lem:indist}, a policy in $\Pid$ uses only $\Bdir$ and stops when the posterior first satisfies $\max(p,1-p)\ge 1-\varepsilon$. This is a sequential hypothesis test between $q=0$ and $q=1$ through a channel with divergence $\KL=\Theta(\gamma^2)$. The change-of-measure converse for sequential tests, obtained from Wald's identity for the log-likelihood ratio, gives
\[ \mathbb{E}[N]\cdot \KL \ \ge\ \mathrm{kl}(1-\varepsilon,\varepsilon)\ \ge\ (1-2\varepsilon)\ln\tfrac{1-\varepsilon}{\varepsilon}, \]
where $\mathrm{kl}(a,b)=a\ln\frac{a}{b}+(1-a)\ln\frac{1-a}{1-b}$ is binary relative entropy. Hence any probe-only stopping rule satisfies $\mathbb{E}[\text{myopic cost}]=\mathbb{E}[N]\ge (1-2\varepsilon)\ln\frac{1-\varepsilon}{\varepsilon}/\KL=\Omega(\log(1/\varepsilon)/\gamma^2)$. Under a uniform prior, stopping at posterior confidence at least $1-\varepsilon$ gives Bayes error at most $\varepsilon$. The average of the two per-hypothesis errors is therefore at most $\varepsilon$, so each error is at most $2\varepsilon$; applying the converse at level $2\varepsilon$ changes only constants absorbed by the $\Omega(\cdot)$ notation.

\emph{Ratio.} Therefore,
$\mathbb{E}[\text{myopic cost}]/\OPT \ge \Omega\big(\log(1/\varepsilon)/(\gamma^2((d+1)c_b+1))\big)$.
Hold $d$, $c_b$, and $\varepsilon$ fixed and let $\gamma\to 0$. The bound diverges, so $\gamma$ can be chosen small enough that the ratio exceeds any prescribed $\rho$.
\end{proof}

\paragraph{Capped instance (restated).} Instance $I'(d,\gamma,c_b,\varepsilon)$ is identical to $I$ except that a nuisance bit $\nu$ is drawn once per episode with $\Pr[\nu=1]=\beta$, where $\beta$ satisfies $\max(\beta,1-\beta)=1-\varepsilon'$ for some $\varepsilon'>\varepsilon$. The direct probe returns $o$ with $\Pr[o=q\oplus\nu]=(1+\gamma)/2$, whereas $\Bprec$ still returns $q$ exactly.

\begin{proof}[Proof of Proposition~\ref{prop:fail}]
By Lemma~\ref{lem:indist}, a policy in $\Pid$ does not take the first build while $\Bdir$ has positive within-horizon value, so it remains confined to $\Bdir$. Repeated direct probes identify $q\oplus\nu$ but not $q$ separately. Because $\nu$ is unobserved and is conditionally independent of the direct-probe channel given $q\oplus\nu$, the posterior on $q$ converges to $\Pr[q=1\mid q\oplus\nu=v]\in\{\beta,1-\beta\}$. Its maximum is $1-\varepsilon'<1-\varepsilon$. Thus the supremum confidence attainable from direct probes lies strictly below the goal threshold, and the belief never enters $G$. As the per-step information about $q$ approaches zero, the build chain can become tied with a zero-information non-constructive action. The no-build-on-ties convention keeps the policy off the chain. It therefore either stops below the threshold or continues probing below the threshold indefinitely. The optimal policy builds the chain and queries $\Bprec$, which bypasses $\nu$ and reaches confidence $1$ at cost $(d+1)c_b+1$.
\end{proof}

\begin{proof}[Proof sketch of Proposition~\ref{prop:explore}]
In $I_N$, each chain position has $N$ repeatable distractor builds, so at least $N+2$ actions remain available and $\OPT \le (d{+}1)c_b + 1$ is unchanged. The goal can be reached through either construction or direct probing. For the construction route, Lemma~\ref{lem:indist} implies that the $\Pid$ component never selects $b_1$ at the initial frontier. The true build is therefore chosen only by exploration, with probability at most $\varepsilon_0/(N{+}2)$ per step. Consequently,
$\Pr[b_1 \text{ is selected within } t \text{ steps}] \le t\varepsilon_0/(N{+}2)$.
For the direct-probe route, any stopping rule that crosses the confidence threshold requires $W=\Omega(\log(1/\varepsilon)/\gamma^2)$ probes, and interleaved construction only increases cost. Taking $t=\tfrac{1}{2}\min\{(N{+}2)/\varepsilon_0, W\}$ and applying a union bound yields $\Pr[\text{goal by }t]\le \tfrac{1}{2}+o(1)$. Hence
$\mathbb{E}[\text{cost}] \ge (c_{\min}/2)\min\{(N{+}2)/\varepsilon_0,W\}$,
where $c_{\min}=\min(1,c_b)$. Letting $N\to\infty$ and $\gamma\to 0$ jointly makes the ratio to $\OPT$ exceed any prescribed $\rho$. A complete proof of the probe-route step requires a finite-time, high-probability converse for sequential testing rather than only Wald's expectation bound. Results of the required change-of-measure form are standard \citep{kaufmann2016}, but we do not derive the specialized bound for this binary channel here; accordingly, this argument is presented as a proof sketch.
\end{proof}

\section{Reproducibility}\label{app:repro}

All experiments use chain-structured capability graphs. Therefore, $\hcap$ is evaluated in the exact closed form from \S\ref{sec:cgplan}, and the general directed-Steiner approximation is not used. The testbed, baselines, regime sweep, quantifier grid, dynamic proposer, distractor and no-gating controls, and estimator-stability analyses are implemented in a compact package containing six core modules and analysis scripts. The package regenerates every table in \S\ref{sec:experiments}; it is available from the authors and will accompany the peer-reviewed version.

Greedy EIG estimates one-step information gain by nested Monte Carlo, using 300 outer samples per candidate by default and 80--120 in selected multi-world sweeps. The transition cells in Table~\ref{tab:gamma} use 300 samples. At the Result~1 setting, greedy EIG remains at 0/10 success for $n_{\mathrm{outer}} \in \{60,120,300\}$. Within the transition band, where candidate scores are nearly tied, success depends on estimator variance; Table~\ref{tab:gamma} therefore reports the estimator setting. Information-score ties are resolved in favor of the lower-cost action and never in favor of a zero-information build. The $H$-step planner commits to construction only when the precision probe lies within its horizon; otherwise, it selects the greedy-EIG action. Selected actions are simulated against the true circuit. Success proportions use Wilson 95\% intervals, and costs and counts use normal-approximation 95\% intervals across worlds.

The exploration, ablation, and full-width-lookahead analyses in \S\ref{sec:explore-abl} are implemented in \texttt{exp\_review.py}. The initial-cost, empirical-admissibility, and expenditure-allocation analyses in \S\ref{sec:audit} are implemented in \texttt{exp\_audit.py}; the informative-probe threshold is $\tau=0.1$ nats, the greedy allocation analysis uses $n_{\mathrm{outer}}=60$ based on the stability check, and the $\varepsilon$-greedy allocation comes from the closed-route cell. The rollout-search comparison in \S\ref{sec:uct} is implemented in \texttt{exp\_pomcp.py}, with goal reward 100, exploration constant 30, rollout depth 18, average backup, and per-world seeds. Evaluation criteria were fixed before these runs. The $\varepsilon$-greedy wrapper composes with any selector, ablations change only the heuristic in the CG-Plan selection rule, and full-width lookahead enumerates determinized action sequences using root-cached discrimination values.

\bibliographystyle{plainnat}
\bibliography{refs}

\end{document}